\documentclass{article}

\usepackage{microtype}
\usepackage{graphicx}
\usepackage{booktabs}
\usepackage{hyperref}

\usepackage[accepted]{eiml_icml2026}
\hypersetup{hypertexnames=false}
\makeatletter
\renewcommand{\ICML@appearing}{\textit{2nd Workshop on Epistemic Intelligence in Machine Learning (EIML@ICML 2026), Seoul, South Korea. Copyright 2026 by the author(s).}}
\makeatother

\usepackage{amsmath}
\usepackage{amssymb}
\usepackage{amsthm}

\theoremstyle{plain}

\theoremstyle{remark}

\icmltitlerunning{Where Cross-Question Structure in LLM Errors Comes From}

\begin{document}
\raggedbottom

\twocolumn[
\icmltitle{Where Cross-Question Structure in LLM Errors Comes From:\\
Shared Seeds, Shared Prompts, and Answer-Letter Preference}

\icmlsetsymbol{equal}{*}
\begin{icmlauthorlist}
\icmlauthor{Izhar Ali}{rowan}
\end{icmlauthorlist}
\icmlaffiliation{rowan}{Rowan University}
\icmlcorrespondingauthor{Izhar Ali}{aliz@rowan.edu}

\vskip 0.3in

\begin{abstract}
Disagreement between language models is often read as missing
knowledge. Questions that need the same missing
knowledge should then be right or wrong together. We test for
this structure with the eigenvalues of the questions' correlation matrix.
Within one model, independently drawn answers show no detectable
structure; it appears when
all questions in one pass over the benchmark share a random seed or
few-shot prompt. Across 24 models, the largest structure beyond overall
skill is a preference among answer letters: when the options are
rotated, it follows the letter, not the question, and it depends on one
model, Llama-3.2-1B; without it only a smaller excess remains. Counting a
question correct only if right in every rotation removes the part
that acts on guesses. What remains is a marginal excess at the 5\% line,
tied post hoc to two near-duplicate model pairs and to Llama-3.2-1B.
\end{abstract}
]

\printAffiliationsAndNotice{}

\section{Introduction}
\label{sec:intro}

Sampling-based methods such as self-consistency, SelfCheckGPT and
semantic entropy sample a language model several times on the same
question and read disagreement among the answers as uncertainty about
the answer~\citep{selfconsistency2023,selfcheckgpt2023,semanticentropy2024}.
In a deep ensemble, disagreement among the members is read as epistemic
uncertainty: the part that comes from missing knowledge, rather than from
noise that would remain even if the model were
known~\citep{lakshminarayanan2017,kendall2017uncertainties,huellermeier2021aleatoric}.
Recent scores add disagreement between models to self-consistency and
treat it as epistemic~\citep{hamidieh2025}.

This reading makes a prediction. A model that lacks some piece of
knowledge should fail every question that needs it, so across an
ensemble those questions should be right or wrong together. We call that
tendency \emph{cross-question structure}. Related structure is already
documented: different models often choose the same wrong
answer~\citep{correlatederrors2025}, and a few latent factors summarize
benchmark scores~\citep{burnell2023structure,ruan2024observational,iqtest2025}.
We ask whether cross-question structure really comes from shared
knowledge.

Two other sources can produce the same structure, and both are easy to
miss. The first is the evaluation pipeline. Within one model, knowledge
is fixed, so each question has the same chance of a correct answer on
every pass (pass~$k$ is the model's $k$th answer to every question, not
pass@$k$). Knowledge therefore cannot make two questions right on the
same passes and wrong on the others; only something the questions share
inside a pass can. A shared random seed is one example. All questions in
the pass then receive the same random numbers from the sampler, so a
draw that favors the letter ``B'' at one generation step favors it in
every question whose answer letter is emitted at that step. A shared
few-shot prompt is another, since few-shot examples shape
answers~\citep{zhao2021calibrate,min2022rethinking}.

The second source is a letter preference of the model. Llama-3.2-1B is accurate when
the correct answer is placed at ``A'', and below the 0.25 of a random
guess when it is placed at the other letters (Fig.~\ref{fig:example}).
It answers ``A'' in 81\% of its replies in the original option order, and a know-or-guess model fitted to its
answers (Sec.~\ref{sec:method}) puts 0.90 of its guesses on ``A'' (App.~\ref{app:ensemble}). Questions whose correct answer is ``A'' then look like
questions it knows, and across an ensemble those questions are right or
wrong together. That is the pattern shared missing knowledge would
produce. Letter preferences of single models (also called selection bias) are already
documented~\citep{zheng2024selectors,pezeshkpour2024order,ding2026agree}.
Rotating the options separates a letter preference from a knowledge gap. A knowledge gap stays with
the question. A letter preference stays with the answer's letter.

\begin{figure}[t]
\centering
\includegraphics[width=\columnwidth]{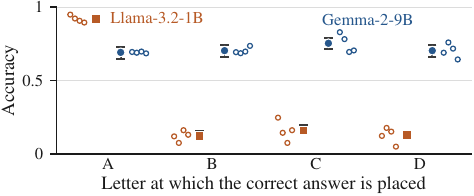}
\vskip -0.08in
\caption{Llama-3.2-1B is accurate where the answer is placed at ``A'',
in each of the four rotations; Gemma-2-9B is about equally accurate at every
letter. Filled: all 500 questions, with the options rotated so that the
correct answer is at that letter; bars: Wilson 95\% CI. Open: each of the four
rotations alone (about 125 questions).}
\label{fig:example}
\end{figure}

We test the two sources separately. Within one model, we compare
independent draws with passes that share a seed or a few-shot prompt
(Sec.~\ref{sec:iid}). Across 24 models, we rotate the options to new
letters, so the gold letter (the letter of the correct option) changes
and the question and options do not
(Sec.~\ref{sec:across}), and we count a question correct only when the
model is right in every rotation (\emph{all-4-correct};
Sec.~\ref{sec:all4}). All three tests use the \emph{correctness
matrix}: 1 if an answer is right, 0 if wrong. Structure is an eigenvalue
of the correlation matrix of its questions that is larger than a null
model allows:
independent answers within one model, and, across models, one skill per
model (a unidimensional item response model). Spectral summaries of
benchmark dimensionality exist~\citep{benchscope2026}.
We add that the correctness matrix has no letters, yet the second
eigenvector of its correlation matrix follows a letter preference;
without one model, its eigenvalue is within the range of one skill
(post hoc; Sec.~\ref{sec:across}).

\begin{figure*}[t]
\centering
\includegraphics[width=\textwidth]{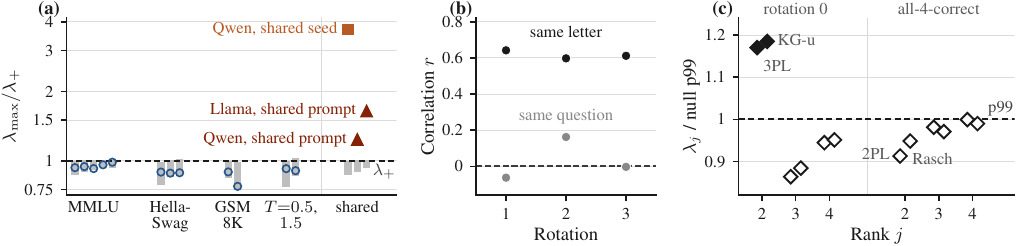}
\vskip -0.05in
\caption{\textbf{(a)}~Within one model: $\lambda_{\max}/\lambda_+$ (log
axis) of each run; grey bar: 5--95\% of its simulated distance from the edge ($z$,
App.~\ref{app:within}), on this axis; circles:
independent draws; filled: an eigenvalue above $\lambda_+$.
\textbf{(b)}~Across models, rotations 1--3: correlation over questions
of each question's eigenvector-2 loading with the mean rotation-0 loading
of the questions that had its new gold letter (same letter), or with its
own rotation-0 loading (same question). \textbf{(c)}~$\lambda_2$--$\lambda_4$
of the 24-model matrix in rotation 0 and all-4-correct, divided by the
99th percentile (1000 draws) of two adequate one-skill nulls (rotation 0:
3PL left and KG-u, uniform know-or-guess, right; all-4-correct: 2PL
left, Rasch right); filled: above it. Other nulls:
App.~\ref{app:ensemble} (soft-preference: not rejected).}
\label{fig:main}
\end{figure*}

\section{A spectral test for cross-question structure}
\label{sec:method}

We record the answers in a correctness matrix $\mathbf{C}$. Entry
$c_{ki}$ is 1 if row $k$ answers question $i$ correctly and 0 otherwise.
Within one model a row is a pass; across models a row is a model. We keep
only questions that vary: within a model, the $N_b$ \emph{borderline}
questions, whose pass rate $f_i$ (the share of passes that answer
question $i$ correctly) lies between $\varepsilon=0.05$ and
$1-\varepsilon$; across models, the \emph{split} questions, which some 
models answer correctly and others do not.

Let $\mathbf{R}$ be the matrix of Pearson (phi) correlations between the
0/1 columns of the kept questions. If the
questions were independent, every true correlation would be 0 and every
eigenvalue would be 1. With a finite number of rows, the measured
correlations are noisy, and the noise spreads the eigenvalues around 1.
Marchenko and Pastur showed that the eigenvalues then fill a range
whose upper edge is $\lambda_+=(1+\sqrt\gamma)^2$, where $\gamma$ is the
number of questions divided by the number of
rows~\citep{marchenkopastur1967}; as both numbers grow, the largest
eigenvalue of such a matrix of independent entries converges to that
edge~\citep{baisilverstein2010,jiang2004correlation}. More questions per
row give a higher edge. For
Llama-3-8B's 173 borderline questions and 100 passes, $\gamma=1.73$ and
$\lambda_+=5.36$. An eigenvalue above that edge is evidence that some group of questions
is right or wrong together more often than noise allows. Its eigenvector gives each question's
\emph{loading}, how strongly it belongs to the group, and each
row's \emph{score}, its projection.

\paragraph{Within one model,} knowledge is fixed, so only the pipeline
can make questions right or wrong together (Sec.~\ref{sec:intro}). This
test checks the pipeline, not
knowledge, and its null model is independence. We count the eigenvalues
above the Marchenko--Pastur (MP) edge $\lambda_+$ (the \emph{MP count}). Because $\lambda_+$ is a limit
for large matrices, we also compare how far the largest eigenvalue lies
from $\lambda_+$, in units of its typical fluctuation, with the same
distance in 1000 simulated matrices of the same size, in which each
question keeps its pass rate and answers are drawn independently
(App.~\ref{app:stats}).

\paragraph{Across models,} independence is the wrong null model. A strong
model is right on many questions and a weak one on few, so any two
questions are correlated through skill alone; this gives one large
eigenvalue, $\lambda_1$. The null model is a one-skill item response model. Model $m$
has skill $\theta_m$, question $i$ has difficulty $b_i$ and discrimination (slope) $a_i$,
and the probability of a correct answer is $\sigma(a_i(\theta_m-b_i))$,
with $\sigma$ the logistic function. This is the two-parameter logistic
(2PL) model~\citep{lordnovick1968}. With 0/1 answers, skill alone can
also produce a second large eigenvalue, the \emph{difficulty factor} of
binary items~\citep{ferguson1941,mcdonald1974}: because the response
curve is non-linear, a 0/1 question correlates more strongly with
questions of similar difficulty. In the
all-4-correct matrix of Sec.~\ref{sec:all4} (391 split questions, 24
models), $\lambda_+=25.4$, yet one-skill simulations
have a median $\lambda_2$ of 28.2. We fit the model,
simulate datasets from it, and say that $\lambda_j$, the $j$th
largest eigenvalue (rank $j$), \emph{exceeds} the null when it lies above
the 99th percentile of its simulations.

Two nulls ignore guessing: the 2PL, and Rasch, which gives every
question the same discrimination~\citep{rasch1960}. Because a model that does not know the
answer is still right a quarter of the time, we also fit the
three-parameter logistic (3PL), with its guessing parameter, the lowest
chance of a correct answer, fixed at 0.25, and knowledge-or-random-guessing
(\emph{know-or-guess}) nulls. In the uniform one, KG-u, a model that does not
know the answer guesses one of the four letters with equal probability
(App.~\ref{app:ensemble}). A null is \emph{adequate} when the
observed $\lambda_1$, the skill eigenvalue, falls between that null's 1st
and 99th percentiles; a null that cannot reproduce skill cannot judge
what lies beyond it.

Because structure beyond skill can appear in $\lambda_2$, $\lambda_3$
or $\lambda_4$, the \emph{max-statistic test} checks all three:
it compares the largest of their percentiles, each among same-rank
simulations, with that maximum in each simulation; its $p$-value is
$p_M$.

\paragraph{Rotations.} Each model answers every question four times, with
the options moved zero (rotation 0, the original order), one, two or
three letters later, cyclically. We count a question correct only if it
is right in all four rotations: all-4-correct, the
CircularEval rule of~\citet{mmbench2024}. A guess with any fixed letter preference is then correct
with probability at most $1/256$ on every question, because the correct
answer visits each letter once and the product of the four letter
probabilities is at most $(1/4)^4$.

\paragraph{Runs.} Within one model there are twelve
\emph{independent-draw runs} on MMLU~\citep{mmlu2021},
HellaSwag~\citep{zellers2019hellaswag} and GSM8K~\citep{gsm8k2021}, in
which each answer has its own seed, as in self-consistency
(Table~\ref{tab:runs}). Each run has $K=30$ to 100 passes. One \emph{shared-seed run} gives all
questions of a pass one seed, and two \emph{shared-prompt runs} draw five MMLU
development questions (the few-shot \emph{exemplars}) for each pass and place them, with their answers,
before every question of that pass. Across models, 24 models from ten
families answer the 500 MMLU questions in all four rotations,
greedily. A reply that cannot be read as a single answer is \emph{unparseable}
and scored wrong. Such replies are 1.3\% of the 48{,}000 ensemble replies but
occur in 217 of 500 questions
(App.~\ref{app:data}).

\section{Results}
\label{sec:results}

\subsection{Within one model: detectable structure only from a shared seed or prompt}
\label{sec:iid}

Independent draws show no cross-question structure that the test
detects. In all 12
independent-draw runs the MP count is zero, and the distance of the
largest eigenvalue from $\lambda_+$ lies between the 5th and 95th
percentiles of its simulations (Fig.~\ref{fig:main}a; App.~\ref{app:stats}).
The test detects weak structure spread over many questions but not a
small cluster. A factor shared by the questions of a pass, with latent
correlation 0.05 (0/1 correlation about 0.025), is detected in every
simulation at Llama-3-8B's 173 borderline questions and in 47\% at
Qwen2.5-7B's 43. A factor on only nine questions, picked at random or
from one subject, is detected in at most a quarter of simulations, even
at latent correlation 0.2 (App.~\ref{app:power}).

Sharing a seed or a prompt among the questions of a pass creates
structure. In all three such runs the largest eigenvalue lies far above
every simulation, and the MP count is two to four.
With a shared seed, the cause is the sampler's shared random numbers
(Sec.~\ref{sec:intro}; App.~\ref{app:mechanism}), and in Qwen2.5-7B's
run the gold letter explains most of the variance of the top
eigenvector's loadings ($\eta^2=0.73$, permutation $p=0.0002$). With a
shared prompt, Llama-3-8B copies the exemplars' answer letters, and its
top eigenvector also groups by gold letter ($\eta^2=0.53$). Qwen2.5-7B
does not copy them, and its top eigenvector does not group by gold
letter; instead, passes that share an exemplar agree more on borderline
questions than passes that share none ($p=0.0005$). Within one model,
each detected structure thus comes from the pipeline, not from knowledge.

\subsection{Across models: the second eigenvector follows the answer letter}
\label{sec:across}

In rotation 0, the first eigenvector of the 24-model matrix is skill: the models' scores on it correlate with
their accuracy at $r=0.94$. Three nulls with guessing are adequate: the
3PL, KG-u, and the letter-aware null, a know-or-guess null, fitted to rotation 0, in
which a guessing model follows its letter preference. $\lambda_2$,
49.9 on 460 split questions, exceeds the 99th percentile
of all three (Fig.~\ref{fig:main}c), no other eigenvalue does, and the
max-statistic test rejects them ($p_M=0.002$).

The second eigenvector follows the answer letter, not the question. Its
loadings group by gold letter, and the models' scores on it rise with
their share of ``A'' answers ($r=0.76$) and fall with their share of
``D'' answers ($r=-0.90$). Within one rotation, letter and content
cannot be told apart; across rotations they can. Suppose a question's correct answer moves from ``B''
in rotation 0 to ``C'' in rotation 1. If its loading comes from content,
it matches the question's own rotation-0 loading; if from a letter
preference, it matches the mean rotation-0 loading of the questions whose answer was at ``C''. In rotations 1--3 the
loadings follow the letter ($r=0.60$ to $0.64$) and barely follow the
question ($r=-0.06$ to $0.16$; Fig.~\ref{fig:main}b). The second eigenvector is a
letter effect, not a content effect.

A fitted letter preference accounts for part of this effect, not all of
it. The \emph{soft-preference} null, a fourth adequate null, lets
that preference act on every answer, not only on guesses. Against it,
the max-statistic test does not reject ($p_M=0.079$), but $\lambda_2$
alone is still unlikely ($p=0.010$), and the null misses how often
models are right in exactly one or all four rotations
(App.~\ref{app:ensemble}).

Post hoc, one model, Llama-3.2-1B, lifts $\lambda_2$ past the 3PL
null. Removing it
lowers $\lambda_2$ from 49.9 to 35.0, below the 99th percentile of the
3PL null on the other 23 models; removing any other single model leaves
it between 47.1 and 52.4. Without it, some structure remains: the
letter grouping moves to eigenvector 3 ($\eta^2=0.21$), and the
3PL null is still rejected through $\lambda_4$ ($p_M=0.003$; App.~\ref{app:ensemble}).

\subsection{After rotation: a marginal excess tied to a few models}
\label{sec:all4}

All-4-correct cancels a letter preference that acts only on guesses. In
the all-4-correct matrix (391 split questions), $\lambda_2$ lies below the 99th percentile of both
adequate one-skill nulls, the 2PL and Rasch (Fig.~\ref{fig:main}c). Because the
difficulty factor lifts the nulls' $\lambda_2$ above $\lambda_+$, a
second skill could hide below their 99th percentile. In a post hoc simulation, a second skill as
strong as the first, planted on the 125 kept STEM questions, is
detected by the max-statistic test in over 90\% of datasets; one half as strong, in about a fifth (App.~\ref{app:ensemble}).

What remains is a marginal excess on $\lambda_3$ and $\lambda_4$. The
max-statistic test rejects both nulls ($p_M=0.026$ for the 2PL, $0.041$
for Rasch), but three facts make the excess weak. No eigenvalue exceeds
the 99th percentile of both nulls (Table~\ref{tab:fw}). The test rejects
slightly too often under one skill (up to 8.7\% of datasets at the 5\%
level); recalibrated post hoc under 2PL and Rasch truth, the $p_M$ range
from 0.017 to 0.057. Dropping the 217 questions with an unparseable reply
ends the rejection. The excess sits at the
5\% line.

The excess is tied to a few models. Before running the max-statistic test, we listed the
\emph{sibling groups}, successive releases or fine-tunes in one family
at one size: Phi-3/3.5-mini, Llama-3/3.1-8B, Qwen2/2.5-7B, and
Mistral-7B v0.2/v0.3 with Zephyr-7B. Keeping one model from each group,
in any of the 24 ways, ends the rejection, while removing five other
models at random keeps it in all 24 draws.
Post hoc analyses narrow this to two pairs and also point to
Llama-3.2-1B. Removing
one of Phi-3-mini and Phi-3.5-mini and one of Llama-3-8B and
Llama-3.1-8B raises $p_M$ above 0.2 under both nulls, which none of 190
removals of two other models does, and 2PL residual correlations rank
these two pairs first and second of 276. Removing only Llama-3.2-1B, the
model with the most skewed letter shares, ends the rejection
($p_M=0.059$ for the 2PL, $0.089$ for Rasch), though 12 of 19
removals of it with another non-pair model still reject under the 2PL,
so its role is uncertain. Its preference may partly remain after
all-4-correct if it also tilts partly known answers, which no adequate
null here tests (App.~\ref{app:ensemble}).

\section{Discussion}
\label{sec:discussion}

Cross-question structure in correctness is not, by itself, evidence of
shared missing knowledge. In our data it came from the pipeline within a
model and, in its largest part, from a letter preference across models,
and what remained after rotation is marginal and tied to a few models. A letter
preference, even one carried by a single model, can create an
eigenvector that an ensemble would read as missing knowledge. Scores that treat cross-model
disagreement as epistemic~\citep{hamidieh2025} should be checked for letter
preference; peer agreement can still flag errors, as two large peers do
for Qwen2.5-7B better than 100 of its own samples (post hoc;
App.~\ref{app:selective}). A shared seed keeps each question's answer
distribution but feeds the same random numbers to all questions of a
pass, so comparing questions can detect it when the number correct per pass
does not.

Four limits remain. First, the marginal excess is
fragile: particular removals end it (Sec.~\ref{sec:all4}), and the 217
questions with an unparseable reply may carry part of it, though
dropping them leaves only 236 split questions. Second, several analyses were chosen after seeing
the data: the rule for an adequate null, the removals that isolate the
two pairs, and the analyses of test size, power and the difficulty
factor (App.~\ref{app:ensemble}); why the models in a pair fail
together, through shared knowledge gaps from common training data,
contamination or other causes, was not tested. Third, the scope is
narrow: in the ensemble, architecture, family, scale and training data
vary together; the within-model runs use instruction-tuned models of
1.7B to 8B, and the across-model runs 500 MMLU questions. Fourth, the
test can miss structure carried by a few models and anything outside
the correctness matrix, such as log-probabilities, and no threshold is
corrected for the nulls and scoring rules tried.

Before an ensemble's errors are read as missing knowledge, each sampled
answer needs its own seed. The options should be rotated, a question counted
correct only if right in every rotation, and any remaining
structure tested against a fitted one-skill model and checked for
near-duplicate models.

{\footnotesize
\bibliography{references}
\bibliographystyle{icml2026}
}

\appendix
\renewcommand{\thefigure}{A\arabic{figure}}
\renewcommand{\thetable}{A\arabic{table}}
\renewcommand{\theHfigure}{A\arabic{figure}}
\renewcommand{\theHtable}{A\arabic{table}}
\setcounter{figure}{0}
\setcounter{table}{0}

\section{Data and scoring}
\label{app:data}

This appendix describes the questions, prompts and scoring, and checks
five features of the data.

\paragraph{Questions.} MMLU: ten random questions from each of the
first 50 subjects in alphabetical order. HellaSwag: 200 validation
items, sampled within each activity label. GSM8K: the first 100 test
problems. In the shared-prompt runs, each pass draws its five exemplars from the 285 MMLU development
questions, so two passes can share an exemplar; both runs use the same
exemplar sets.

\paragraph{Prompts and decoding.} Each model receives one user message
in its own chat template (raw text for Phi-2, which has none), checked
against an independent reference tokenization. Generation uses vLLM 0.23.0 in bfloat16, with at most 8 new
tokens for MMLU and HellaSwag and 256 or 1024 for GSM8K; the
within-model runs sample at temperature $T$ without top-$k$ or top-$p$
truncation. Every answer (pass $k$, question $i$) has its own seed,
except in the shared-seed run, and vLLM gives each seeded request its
own random-number generator, so an answer's noise does not depend on its
position in a batch. The 24-model runs decode greedily with vLLM's
deterministic engine settings; generation still differs slightly between
engine starts.

\paragraph{Scoring.} A reply counts as letter X if it
begins with X (optionally after ``Answer:'', a parenthesis or bold
markers) without negating or hedging, or if its first line says ``the
answer is X'' or an equivalent. Any other reply, including one
that names several letters, is unparseable and scored wrong. A GSM8K
reply is correct when the number after its last ``\#\#\#\#'', or else
its last number, equals the gold answer.

\paragraph{Unparseable replies.} Unparseable replies are rare within one
model, at most 1.65\% of a run's replies (Llama-3-8B on HellaSwag), and
change no MP count: in the eight runs that have any, dropping every
question with one, or imputing each such reply by the question's mean,
leaves the MP count at zero in the seven independent-draw runs and at
two in Llama-3-8B's shared-prompt run. Across models they are more
frequent: in rotation 0, Mistral-7B v0.2 has 12.4\% and every other
model at most 9.6\%, and 217 of the 500 questions have one in some
rotation (Table~\ref{tab:fw}).

\paragraph{Truncation on GSM8K.} Truncation changes which answers are
correct but does not couple two questions of a pass, so a truncated run
still tests independence. At 256 new tokens, 55\% of Qwen2.5-7B's GSM8K
replies are cut off, and its accuracy is 51.0\%, against 90.3\% at 1024
tokens.

\paragraph{Replicates.} GPU nondeterminism changes the ensemble little, but can
move a result at a threshold. We collected the ensemble twice; the paper
reports replicate A. Replicate B differs
in 6 of the 48{,}000 chosen options and in eigenvalues by at most 0.081,
and its all-4-correct $\lambda_4$, 20.82, lies just above the 2PL
rank-4 99th percentile of 20.78 (1000 draws), so in both replicates
$\lambda_4$ sits on that percentile.

\paragraph{Near ties.} A \emph{cell} is one model's answer to one
question in one rotation. It is a \emph{near tie} when its two most
probable first-token letters differ by less than 0.13 nats in
log-probability in either replicate; 5.0\% of cells are. Replacing each
near tie by the majority correctness over the four rotations changes no
number of ranks beyond skill (App.~\ref{app:ensemble}), in rotation 0
or on all-4-correct. Dropping instead the 297 questions with a rotation-0 near tie
leaves 203; $\lambda_2$ then exceeds the 3PL null but not the
letter-aware null, both adequate.

\paragraph{Option references.} Options such as ``both A and B'' change
meaning under rotation. Removing the 23 questions whose options refer to
other options changes no number of ranks beyond skill in
rotation 0 or on all-4-correct.

\section{The within-model test}
\label{app:within}

This appendix explains the statistic of the within-model test, checks
its calibration on every run, and measures its power.

\paragraph{The statistic.} An eigenvalue test pools many weak
correlations. If $n$ questions all have pairwise correlation $\kappa$,
the largest eigenvalue of their true correlation matrix is
$1+(n-1)\kappa$, which grows with $n$. A test of that eigenvalue can detect a correlation smaller than the noise of one sample
correlation, about $1/\sqrt{K}$ for $K$ passes. Each simulated matrix
re-applies the borderline filter, so it has its own $N_b$ and edge
$\lambda_+$. To put all matrices on one scale we use
$z=(\lambda_{\max}-\lambda_+)/\sigma_{\mathrm{TW}}$, with
$\sigma_{\mathrm{TW}}=(1+\sqrt\gamma)(1+1/\sqrt\gamma)^{1/3}K^{-2/3}$,
the Tracy--Widom scale of $\lambda_{\max}$ for Gaussian
data~\citep{johnstone2001}. For correlation matrices the limit is
proved, as far as we know, only without centering or for Gaussian
entries~\citep{baopanzhou2012,pillaiyin2012}. Our columns are centered
Bernoulli variables with unequal pass rates, so $z$ is calibrated by simulation.

\subsection{Calibration}
\label{app:stats}

\begin{table}[t]
\caption{The test on every within-model run ($\varepsilon=0.05$).
$\lambda_+$: MP edge. MC pct.\ $z$: Monte Carlo percentile of $z$ (App.~\ref{app:within}) among 1000
simulated matrices of independent answers with the same pass rates.
MP count: number of eigenvalues above $\lambda_+$. GSM8K runs use
$T=1$.}
\label{tab:runs}
\vskip 0.06in
\centering
\begin{small}
\setlength{\tabcolsep}{4pt}
\begin{tabular}{@{}lrrrrc@{}}
\toprule
\textbf{Run} & $N_b$ & $\lambda_+$ & $\lambda_{\max}$ & MC pct.\ $z$ & \textbf{MP count} \\
\midrule
\multicolumn{6}{@{}l}{\emph{Independent, MMLU} ($K=100$, $T=1$)} \\
Qwen2.5-7B      &  43 & 2.74 & 2.56 & 60.1 & 0 \\
Mistral-7B-v0.3 &  83 & 3.65 & 3.45 & 51.3 & 0 \\
Llama-3-8B      & 173 & 5.36 & 4.97 & 10.2 & 0 \\
SmolLM2-1.7B    & 406 & 9.09 & 8.74 & 28.6 & 0 \\
Phi-3-mini      & 271 & 7.00 & 6.91 & 83.5 & 0 \\
\multicolumn{6}{@{}l}{\emph{Independent, HellaSwag} ($K=50$, $T=1$)} \\
Qwen2.5-7B      &  16 & 2.45 & 2.19 & 69.9 & 0 \\
Mistral-7B-v0.3 &  65 & 4.58 & 4.05 & 19.7 & 0 \\
Llama-3-8B      &  86 & 5.34 & 4.75 & 11.3 & 0 \\
\multicolumn{6}{@{}l}{\emph{Independent, Qwen2.5-7B} ($K=30$)} \\
GSM8K, 256 tok.  &  62 & 5.94 & 5.32 & 37.9 & 0 \\
GSM8K, 1024 tok. &  19 & 3.22 & 2.50 & 12.0 & 0 \\
MMLU, $T=0.5$    &  25 & 3.66 & 3.39 & 78.5 & 0 \\
MMLU, $T=1.5$    &  85 & 7.20 & 6.53 & 32.7 & 0 \\
\midrule
\multicolumn{6}{@{}l}{\emph{Shared seed or prompt, MMLU} ($K=100$, $T=1$)} \\
Qwen, seed      &  41 & 2.69 & 10.00 & 100 & 4 \\
Qwen, prompt    &  93 & 3.86 & 4.77 & 100 & 2 \\
Llama, prompt   & 264 & 6.89 & 11.31 & 100 & 2 \\
\bottomrule
\end{tabular}
\end{small}
\vskip -0.1in
\end{table}

\paragraph{The null.} The null keeps each question's pass rate and
draws answers independently: given the model and the prompts, each entry
$c_{ki}$ is a Bernoulli variable with $\Pr(c_{ki}=1)=f_i$, the
question's observed pass rate, independent across questions and passes.
The borderline filter acts on each column alone, so the kept columns
stay independent. A $z$ above every draw of its null, as in the three
shared-seed and shared-prompt runs, rejects independence but does not
say whether the dependence lies between questions or between passes.

\paragraph{Number correct per pass.} The number of correct borderline
answers per pass varies as independent draws predict: its observed
variance is 0.61 to 1.09
times the predicted one, and no run is overdispersed against a null that
shuffles each column across passes (smallest $p=0.27$). The two lowest
ratios have underdispersion $p=0.039$ and $0.045$, not significant after
a Holm correction over 48 tests (12 runs; over- and underdispersion, and
trends over the pass index in the number correct and the top-eigenvector
score). This check cannot detect a dependence that cancels in the number
correct, such as the shared seed's (App.~\ref{app:mechanism}).

\paragraph{Borderline threshold.} The verdicts do not depend on the
borderline threshold. At $\varepsilon=0.01$, 0.02, 0.05, 0.10, 0.15 and
0.20, with the filter re-applied to the data and to each of the 1000
null matrices, no independent-draw run has an eigenvalue above its edge
or a $z$ above the 95th percentile of its null, and every shared-seed or
shared-prompt run keeps at least two eigenvalues above its edge. The six
question sets of a run are nested, so its six tests are not independent;
at $\varepsilon=0.20$ the three smallest sets have 6, 7 and 10
questions, so in those runs the largest thresholds rest on very few questions.

\subsection{Power}
\label{app:power}

\paragraph{Simulation.} We plant a factor of known strength: one random number per pass, shared by a set of loaded questions. The matrices have the pass rates of
two independent-draw MMLU runs, Qwen2.5-7B ($N_b=43$) and Llama-3-8B
($N_b=173$), $K=100$ passes and all 500 questions. Entry $c_{ki}$ is 1
when a hidden standard normal score falls below $\Phi^{-1}(f_i)$, so
each question keeps its pass rate. On loaded questions the score is
$\sqrt{\rho}\,u_k+\sqrt{1-\rho}\,e_{ki}$, with $u_k$ shared by the
loaded questions of pass~$k$ and $e_{ki}$ independent noise; elsewhere
it is $e_{ki}$. Two loaded questions then correlate at $\rho$, the latent
correlation, in their scores and, at these pass rates and averaged
over pairs, at about $\rho/2$ in their 0/1 entries. The factor loads on
all borderline questions (diffuse), on 30 at random, or on 9
chosen at random or from one subject. It is detected when the MP count,
with the filter re-applied, is positive. Each scenario has 500
matrices (2000 at $\rho=0$).

\begin{figure}[tb]
\centering
\includegraphics[width=\columnwidth]{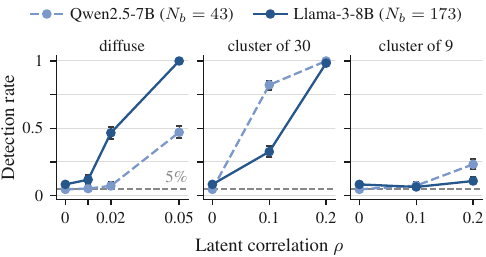}
\vskip -0.08in
\caption{Detection rate of the MP count against the latent correlation
$\rho$ of a factor planted on all borderline questions or on 30 or 9 of
them, at the pass rates of two MMLU runs. Whiskers: Wilson 95\%
intervals.}
\label{fig:power}
\end{figure}

\paragraph{Results.} The test detects diffuse structure and misses
small clusters. A diffuse factor with $\rho=0.05$
is detected in every simulation at Llama-3-8B's 173 borderline questions
and in 47\% at Qwen2.5-7B's 43, because at 173 it loads four times as
many questions (Fig.~\ref{fig:power}). A cluster of fixed size is harder
to detect at 173, where it is a smaller share of the borderline
questions: a 9-question cluster at $\rho=0.2$, the largest strength
simulated, is detected in 23\% of Qwen2.5-7B and 11\% of Llama-3-8B
simulations when chosen at random, and in 24\% and 12\% when drawn from
one subject. Without a factor, the MP count has false-positive rates of
4.6\% and 8.3\%. We did not simulate other values of $N_b$.

\paragraph{Subject structure in the data.} Borderline questions of one
subject correlate no more than other pairs: mean within-subject correlations are $-0.018$ ($p=0.80$) for Qwen2.5-7B and $0.002$ ($p=0.37$)
for Llama-3-8B (MMLU, $T=1$). This direct check complements the power
results, which show that a cluster of nine questions is often missed.

\section{What a shared seed or prompt couples}
\label{app:mechanism}

This appendix shows how a shared seed and a shared prompt couple the
questions of a pass.

\paragraph{Shared noise.} A shared seed gives every question of a pass
the same random numbers at every step, whatever its prompt.
At each step, a seeded request draws one exponential random number per token and emits the token whose probability at temperature $T$,
divided by that number, is largest. This is equivalent to Gumbel-max
sampling, so each answer alone follows the model's distribution.
Coupled
generation and counterfactual sampling use it on purpose~\citep{corvelobenz2025coupled,ravfogel2025gumbel}.

\begin{figure}[tb]
\centering
\includegraphics[width=\columnwidth]{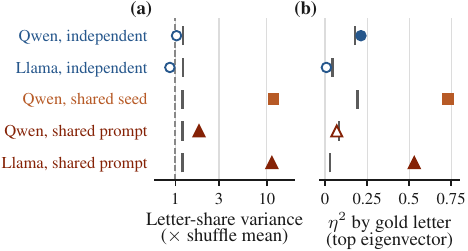}
\vskip -0.08in
\caption{A shared seed or prompt makes a pass's letter shares vary
beyond chance. Of the shared runs, only the seed run and Llama-3-8B's
prompt run also group the top eigenvector's loadings by gold letter.
\textbf{(a)}~Across-pass variance of a pass's letter shares, summed over
letters, divided by the mean of that sum after each question's answers
are shuffled across passes.
\textbf{(b)}~$\eta^2$, the share of the variance of the top
eigenvector's loadings that the gold letter explains.
Grey tick: null 99th (a) or 95th (b) percentile; filled: beyond it.}
\label{fig:mechanism}
\end{figure}

\paragraph{The shared-seed run.} The shared seed moves all answers of a
pass toward the same letters: in Qwen2.5-7B's shared-seed run, the
across-pass variance of a pass's letter shares (its fractions of A, B, C
and D answers), summed over letters, is 11.8 times its mean after each
question's answers are shuffled across passes, against 1.03 in the
independent-draw run (Fig.~\ref{fig:mechanism}a). A pass whose noise
favors one letter gains on questions with that gold letter and loses on
the others, so the top eigenvector's loadings group by gold letter
(Sec.~\ref{sec:iid}), while the gains and losses cancel in the
number correct per pass (observed over predicted variance 0.83).

\paragraph{The shared-prompt runs.} Llama-3-8B copies the exemplars' letters, as
few-shot prompts are known to induce~\citep{zhao2021calibrate}: the
number of a pass's exemplars whose answer is letter $L$ correlates with
the pass's share of answers $L$ at $r=0.35$ to $0.47$. Its letter shares
vary 11.3 times their shuffled mean, and its number correct per pass is
overdispersed (ratio 1.51, $p=0.001$). Qwen2.5-7B does not copy the
letters: its letter shares vary 1.8 times their shuffled mean and do not
follow the exemplars' letters ($|r|\le0.09$; loadings by gold letter
$\eta^2=0.07$). Instead, its passes that share an exemplar agree more on borderline questions than
passes that share none, by 10.7 standard deviations of a null that
permutes the exemplar sets among passes (5.9 for Llama-3-8B; $p=0.0005$
for both).

\paragraph{Gold letter below the edge.} In Qwen2.5-7B's independent-draw
run, whose $\lambda_{\max}$ lies below $\lambda_+$, the top eigenvector's
loadings group by gold letter with $\eta^2=0.21$ ($p=0.023$;
Fig.~\ref{fig:mechanism}b). This is one uncorrected test on an
eigenvector inside the noise range, not interpreted.

\section{The 24-model ensemble}
\label{app:ensemble}

This appendix defines the models, nulls and tests, and gives the
evidence behind Secs.~\ref{sec:across} and~\ref{sec:all4}.

\subsection{Models, nulls and tests}

\paragraph{Roster.} Five Qwen (Qwen2.5 at 0.5B, 1.5B, 3B and 7B;
Qwen2-7B), five Mistral (Mistral-7B v0.2 and v0.3, Zephyr-7B-beta,
Ministral-8B, Mistral-Small-22B), four Llama (Llama-3.2 at 1B and 3B,
Llama-3-8B, Llama-3.1-8B), three Phi (Phi-2, Phi-3-mini, Phi-3.5-mini),
two Gemma-2 (2B, 9B), and SmolLM2-1.7B, OLMo-2-7B, Yi-1.5-6B,
DeepSeek-LLM-7B and Granite-3.1-8B. Phi-2 is a base model; the other 23
are instruction-tuned. Rotation-0 accuracy runs from 35.0\%
(Llama-3.2-1B) to 72.8\% (Gemma-2-9B).

\paragraph{One-skill nulls.} In every null, skill and, where present,
the model's letter preference are the only links between questions. Every simulated draw is conditioned on each kept
column staying split. The 2PL (Sec.~\ref{sec:method}) treats the skills
as fixed and integrates the question parameters over a normal
population, by marginal maximum likelihood; each simulated column draws
its $(a_i,b_i)$ from its posterior in that fit, made on all 500
questions. Rasch is fitted by joint maximum likelihood. In KG-u, a model
either knows a question, with the 2PL probability, in all four rotations
at once, or guesses a letter uniformly in each rotation; KG-u is fitted
to all four rotations and scored like the data. The letter know-or-guess
null (KG-l) instead draws guesses from the model's letter preference
$\pi_m$, its probabilities of guessing A, B, C and D. We estimate
$\pi_m$ from each model's letter shares and accuracy over all four
rotations, assuming that a known answer falls on each letter equally
often across rotations and a guess falls on letter $L$ with probability
$\pi_m(L)$; for Llama-3.2-1B, $\pi_m(\mathrm{A})=0.90$. The letter-aware
null uses the same $\pi_m$ but is fitted to rotation 0 only: model $m$
is correct with probability
$\pi_m(g_i)+(1-\pi_m(g_i))\,\sigma(a_i(\theta_m-b_i))$, where $g_i$ is
the gold letter. The soft-preference (SP) null is defined in
App.~\ref{app:sp}. We also tried a Bock--Aitkin 2PL, with skills
integrated and slopes at most 4, and a fixed-margin null, uniform over
0/1 matrices with the observed margins.

\paragraph{Adequacy.} The adequate nulls differ between the two
scores. In rotation 0 (1000
draws per null) the observed $\lambda_1$ lies at the 76.4th percentile
of the 3PL, 86.7th of the letter-aware, 98.7th of the KG-u and 85.7th of
the SP null, which are adequate, and at the 100th of the 2PL and Rasch.
On all-4-correct (2000 draws) it lies at the 31.5th and 54.2nd
percentiles of the 2PL and Rasch, which are adequate, at the 100th of
the 3PL and at the 0th of the KG-u and SP. KG-l, Bock--Aitkin and
fixed-margin are inadequate for both.

\paragraph{Max-statistic test.} The test uses 2000 new null draws. Its
statistic $M$ is the largest of the percentiles of $\lambda_2$,
$\lambda_3$ and $\lambda_4$ among these draws, ties counted as half.
Each draw's $M_d$ is computed likewise against the other
draws, and $p_M=(1+\#\{d:M_d\ge M\})/2001$; the test rejects at
$p_M<0.05$. Three null draws reach $M=100$, so $p_M$ cannot fall below
0.002, the value in rotation 0. The \emph{rank-2 rule} rejects when
$\lambda_2$ exceeds the null's 99th percentile. For an adequate null,
the \emph{number of ranks beyond skill} counts ranks $2,3,\ldots$
above their 99th percentiles, stopping at the first that is not.

\paragraph{Order of analyses.} The max-statistic test was written down
before its first run, together with the list of sibling groups, the
adequacy rule, a calibration under the SP null and a power study; the
adequacy rule itself had been set after the first analyses. These
choices still depend on earlier analyses: an earlier post hoc joint
statistic had already rejected,
and the separate rank tests and the correlations between model pairs
had already been computed. Removing single models or pairs, and the
reported analyses of size, power and the difficulty factor, came after
the results.

\paragraph{Size (post hoc).} Under a true one-skill model the test rejects slightly
too often, and the adequacy rule can pick the wrong
null. The size of a test is its false-positive rate
when the null model is true. We simulate 300 all-4-correct datasets from
each fitted one-skill model, 2PL and Rasch, refit both nulls to each
(500 draws), and apply the adequacy rule and the max-statistic test
(Table~\ref{tab:fw}). The 2PL test, the max-statistic test against the
2PL null, rejects in 6.7\% of datasets under 2PL truth and 8.7\% under
Rasch truth (Clopper--Pearson 95\% intervals 4.1--10.1\% and
5.7--12.4\%), so under Rasch truth it exceeds the bound of 7.5\% we set
before the run; for this reason Sec.~\ref{sec:all4} reports $p_M$
recalibrated against these datasets, which puts the excess at the 5\%
line. The Rasch test rarely rejects under Rasch truth (0.7\%),
mainly because the refitted Rasch null is adequate in only 23\% of those
datasets. The adequacy rule accepts the Rasch null on every 2PL dataset,
despite wrong slopes, and rejects the correct Rasch null
in most datasets drawn from it. The calibrated $p_M$ is the share of
one truth's 300 datasets with $p_M$ at most the observed value.

\subsection{Rotation 0: the letter eigenvector}
\label{app:sp}

This subsection gives the evidence that eigenvector 2 of rotation 0
follows the gold letter, and tests nulls that contain a letter
preference.

\begin{figure}[tb]
\centering
\includegraphics[width=\columnwidth]{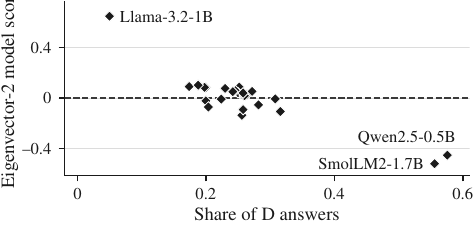}
\vskip -0.08in
\caption{Each model's rotation-0 score on eigenvector 2 of the
24-model matrix against its share of D answers.}
\label{fig:letters}
\end{figure}

\paragraph{Letter evidence.} Eigenvector 2 follows the gold letter: its
loadings group by gold letter ($\eta^2=0.48$, $p=0.0002$), and the
models' scores on it track their share of D answers (permutation
$p=10^{-4}$; Fig.~\ref{fig:letters}). The loadings on eigenvector 1 also
group by gold letter ($\eta^2=0.22$), and those on eigenvectors 3 and 4
more weakly, though beyond chance ($\eta^2=0.087$ and $0.061$,
$p=0.0002$ each). Because gold letters are nearly balanced (118 to 129
per letter), a letter preference barely changes expected benchmark scores.

\paragraph{Where the letter preference acts.} The letter-aware null lets the preference act only on guesses; the 3PL and KG-u nulls have
none. A preference that also tilts answers a model partly knows can, with one skill, put $\lambda_2$ above the letter-aware
and 3PL nulls: in 200 datasets from the SP fit, each null refitted (200 draws), $\lambda_2$ exceeds the letter-aware null's 99th percentile in
60\% and the 3PL null's in 99.5\%.

\paragraph{The soft-preference null.} Model $m$ chooses option $j$ of
question $i$ in a rotation with probability proportional to
$\exp\!\big(b_{ij}+\mathbb{1}[j=g_i](a_i\theta_m+e_{mi})\big)\,\pi_m(L_j)$.
Here $g_i$ is the correct option, $b_{ij}$ is a fitted intercept for
option $j$ with $b_{i g_i}=0$, $L_j$ is the letter at which the rotation
places option $j$, and $e_{mi}\sim N(0,\tau^2)$ is shared by the four
rotations of one model on one question and independent across models.
Skill raises only the correct option, while $\pi_m$ multiplies every option.
The fit is penalized maximum likelihood on the parsed choices of all four
rotations ($\tau=1.17$).

\paragraph{Calibration of the SP null.} Against the calibrated SP null,
$\lambda_2$ alone is unlikely ($p=0.010$), but the max-statistic test
does not reject ($p_M=0.079$); neither verdict is fully validated.
Simulations keep the cells unparseable in the data unparseable, scored
wrong: without that mask, the observed
$\lambda_2$ lies at the 99.8th percentile of the null; with it, at the
95.1st. Refitting inflates the slopes, so tail probabilities are recalibrated on 200 datasets from the fit, each with its own null
of 1000 draws. The calibrated $p$ of rank $j$ is the fraction of those
datasets whose percentile for $\lambda_j$ is at least the observed one.
For $\lambda_2$ in rotation 0 it has Monte Carlo standard error 0.007,
and is 0.015 under the convention $(1+k)/(1+n)$. The calibration is exact only if
each refit inflates the slopes as much as the first, and does not correct the misfit described next. The null is inadequate without
Llama-3.2-1B and on all-4-correct.

\paragraph{Predictive check.} The SP null fails the predictive check we
set before the run, against 500 datasets simulated from the fit: it
reproduces per-model accuracy and letter shares but not the number of
rotations in which a model answers a question correctly. Of the shares
of questions with each count from 0 to 4, per model, only 82.5\% lie
inside the central 98\% of their simulated values (90\% required), and
the observed number of split questions under ``correct in at least two
rotations'' lies above all simulations. The weakest, letter-biased
models have too many questions correct in exactly one rotation, and 21
of 24 models too many in all four (largest excess 0.06 of the
questions).

\paragraph{Without Llama-3.2-1B (post hoc).} Without Llama-3.2-1B, the rotation-0 excess moves from $\lambda_2$ to higher ranks: in 2000
draws, $\lambda_3$ and $\lambda_4$ lie at the 98.5th to 99.9th
percentiles of the 3PL and letter-aware nulls, above the 99th for
$\lambda_4$ under both and for $\lambda_3$ under the letter-aware null. The $\lambda_2$ of the other 23 models (Sec.~\ref{sec:across})
lies below the 3PL null's 99th percentile, 40.1. The max-statistic
test rejects the 3PL null ($p_M=0.003$; Sec.~\ref{sec:across}) and the
letter-aware null ($p_M=0.015$), both through $\lambda_4$. In their 1000 draws, $\lambda_4$ exceeds the
99th percentile by only 0.35 (3PL) and 0.57 (letter-aware), so the
remaining excess is small.

\paragraph{Unparseable replies and sibling models.} Without the 112 questions with an unparseable
rotation-0 reply, the max-statistic test still rejects the letter-aware
null ($p_M=0.013$) and the 3PL ($p_M=0.004$), although $\lambda_2$
falls to the 98.3rd percentile of the letter-aware null, below its 99th,
while it stays above the 3PL null's 99th. Keeping one model from each
sibling group (Sec.~\ref{sec:all4}), in any of the 24 ways, leaves both
nulls adequate, and the test rejects both in every subset.

\paragraph{Option position and letter preference.} Part of a model's self-agreement depends on
option position~\citep{ding2026agree}. In the option-level response
model of \citet{fei2026wrong}, a per-model positional preference adds
little to the prediction of held-out correctness (accuracy 0.840 with
it, 0.838 without). We did not test whether letter preference affects
item-level dimensionality, factor and item response
analyses~\citep{metabench2025,benchscope2026,jeirt2025}, or analyses of
shared wrong answers~\citep{correlatederrors2025}, as it affects these
eigenvectors.

\subsection{All-4-correct: the marginal excess}

This subsection gives the power of the test, the difficulty factor, and
the removals behind the marginal excess.

\paragraph{Power (post hoc).} The test detects a second skill as strong as the
first, but usually not one half as strong. We simulate
from the adequate all-4-correct 2PL fit plus a second skill. As in the
2PL null, each of the 391 kept columns draws $(a_i,b_i)$ from its
posterior. On the 125 of them from the 18 STEM subjects, the logit
$a_i(\theta_m-b_i)$ gains $s\,\bar a\,\mathrm{sd}(\theta)\,z_m$. Here
$z_m\sim N(0,1)$ is model $m$'s second skill, independent of
$\theta_m$. The factor $\bar a=1.71$ is the median slope, and
$\mathrm{sd}(\theta)=1$, so at $s=1$ the second skill matches the first
in spread and median slope. We refit both nulls to each of 200
datasets per $s$ (500 draws; rates in Table~\ref{tab:fw}). At $s=2$
both tests detect it in at least 97\%. The median $\lambda_2$ is 28.2, 28.3
and 33.8 at $s=0$, 0.5 and 1; the observed 29.0 lies near the values
without a strong second skill.

\paragraph{The difficulty factor (post hoc).} On all-4-correct, eigenvector 2 behaves
like the difficulty factor that the one-skill nulls produce: its
loadings follow each question's accuracy across models, the share of the 24 models answering it correctly, which runs opposite to difficulty.
The observed $|r|$ between loading and accuracy, 0.57, lies just below the 2PL null's 5th
percentile and inside the Rasch null's 5--95\% range. In the rotation-0 3PL null the loadings barely follow accuracy (median $|r|$
0.09; observed 0.04). The observed all-4-correct loadings group by
subject ($\eta^2=0.21$, $p=0.001$), but a cubic in accuracy explains
39\% of their variance; after its removal the subject grouping is at
chance ($p=0.24$).

\begin{table}[t]
\caption{Max-statistic test on all-4-correct (reject at $p_M<0.05$),
against the 2PL and Rasch nulls. Pct.: percentile of each eigenvalue in
the null (2000 draws; 2PL $\lambda_4$ is just below its 99th percentile
in the 1000 draws of Fig.~\ref{fig:main}c). Subsets: number that reject, and the range or
maximum of $p_M$; sibling: one model kept from each sibling group;
random: five of the other 15 models removed. Pair removals: one Phi and
one Llama model. Two other models removed: the 190 removals of two
models outside the pairs (fraction that reject, maximum $p_M$).
Unparseable: replies imputed, or questions dropped (1000 draws), either
the 95 on which some model has no wrong parsed rotation and at least one
unparseable one, or all 217 that contain one. Size: rejection rate in
300 datasets simulated from each fitted one-skill model. Calibrated
$p_M$: under 2PL and Rasch truth. Power: max-statistic test / rank-2
rule, 200 datasets per strength $s$ (2PL fit).}
\label{tab:fw}
\vskip 0.06in
\centering
\begin{small}
\setlength{\tabcolsep}{4pt}
\begin{tabular}{@{}lcc@{}}
\toprule
& \textbf{2PL} & \textbf{Rasch} \\
\midrule
Pct.\ of $\lambda_2,\lambda_3,\lambda_4$ & 71.8, 98.5, 99.1 & 94.6, 97.0, 98.6 \\
$p_M$, 24 models & 0.026 & 0.041 \\
Sibling subsets & 0/24 (0.09--0.65) & 0/24 (0.12--0.69) \\
Random subsets & 24/24 ($\le$0.047) & 24/24 ($\le$0.046) \\
Pair removals & 0/4 (0.22--0.69) & 0/4 (0.32--0.38) \\
Two other models removed & 0.90 ($\le$0.099) & 0.74 ($\le$0.133) \\
Unparseable: impute & 0.019 & 0.042 \\
\quad drop 95 & 0.036 & 0.043 \\
\quad drop 217 & 0.23 & 0.12 \\
Size, 2PL truth & 0.067 & 0.067 \\
Size, Rasch truth & 0.087 & 0.007 \\
Calibrated $p_M$ & 0.033, 0.047 & 0.057, 0.017 \\
Power, $s=0.5$ & 0.20 / 0.01 & 0.21 / 0.07 \\
Power, $s=1$ & 0.91 / 0.69 & 0.93 / 0.78 \\
\bottomrule
\end{tabular}
\end{small}
\vskip -0.1in
\end{table}

\paragraph{Sibling models.} Keeping one model per sibling group
(Sec.~\ref{sec:all4}), in any of the $2\cdot2\cdot3\cdot2=24$ ways,
ends the rejection: in each subset of 19 models both refitted nulls (1000 draws) are adequate, and the test rejects in none.
In 24 control subsets that remove five of the other 15
models at random, it rejects in all (Table~\ref{tab:fw}). The subsets overlap, so this is one contrast, not
24 replicates. The removed models are not size-matched to the sibling models, and the 15 include models from the same developers
(Phi-2, Qwen2.5 0.5B--3B, Llama-3.2, Ministral-8B, Mistral-Small).

\paragraph{The two pairs (post hoc).} The excess rests on the Phi and
Llama pairs (Sec.~\ref{sec:all4}): removing one model from each, in all
four ways, gives a larger $p_M$ under
both nulls than any of the 190 removals of two other models, of which 90\% reject under the 2PL and 74\% under Rasch (1000 draws; every
null adequate; Table~\ref{tab:fw}). Three other candidates do not carry
the excess. The Qwen pair: removals that include Qwen2-7B or Qwen2.5-7B
reject more strongly, with mean $p_M$ 0.010 and 0.006 under the 2PL.
The Mistral group: removing Mistral-7B v0.2 and v0.3 (third-largest
residual correlation) leaves $p_M=0.019$, and its other
two removals give $p_M$ of 0.004--0.023. Accuracy: the 15 removals
of two models with rotation-0 accuracy 0.62--0.71, around the pair
models' 0.644--0.684, give $p_M$ at most 0.052 (2PL) and 0.113 (Rasch), below every pair removal.

\begin{table}[t]
\caption{Simple baselines. Top: rotation-0 letter shares of the three
models whose largest share exceeds 0.5; no other model's largest share exceeds 0.35. Bottom:
correlation of model pairs' residuals after the all-4-correct 2PL fit,
with rank among 276 pairs (mean $-0.04$, standard deviation 0.08).}
\label{tab:baselines}
\vskip 0.06in
\centering
\begin{small}
\setlength{\tabcolsep}{4pt}
\begin{tabular}{@{}lcccc@{}}
\toprule
\textbf{Model} & A & B & C & D \\
\midrule
Llama-3.2-1B & 0.81 & 0.04 & 0.09 & 0.05 \\
Qwen2.5-0.5B & 0.23 & 0.06 & 0.12 & 0.58 \\
SmolLM2-1.7B & 0.07 & 0.20 & 0.17 & 0.56 \\
\midrule
\textbf{Model pair} & & & $r$ & \textbf{Rank} \\
\midrule
Phi-3-mini, Phi-3.5-mini & & & 0.40 & 1 \\
Llama-3-8B, Llama-3.1-8B & & & 0.27 & 2 \\
Mistral-7B v0.2, v0.3 & & & 0.20 & 3 \\
Mistral-7B v0.3, Llama-3-8B & & & 0.19 & 4 \\
Mistral-7B v0.2, Zephyr-7B & & & 0.17 & 5 \\
Qwen2-7B, Qwen2.5-7B & & & 0.03 & 37 \\
\bottomrule
\end{tabular}
\end{small}
\vskip -0.1in
\end{table}

\paragraph{Model families (post hoc).} The models'
scores on all-4-correct eigenvector 4 group by family ($\eta^2=0.74$, $p=0.006$; Holm over four eigenvectors,
$p=0.024$). Across models the reference is the one-skill null, not
$\lambda_+$: $\lambda_4$ (20.8) lies below $\lambda_+=25.4$ but at the
2PL null's 99th percentile and carries part of the marginal excess, so
this grouping fits the sibling result.

\paragraph{Llama-3.2-1B (post hoc).} Llama-3.2-1B carries part, not all, of
the excess. Removing it alone makes the KG-u null
adequate ($p_M=0.15$). The sum $\lambda_2+\lambda_3+\lambda_4$, specified with
the max-statistic test and not used for decisions, gives $p=0.023$
(2PL) and $0.016$ (Rasch). Of the 190 removals, 12 of the 19 that include it still reject under the 2PL and 6 under Rasch. Removing it and one model per pair gives $p_M$ of 0.30--0.59 (2PL) and 0.38--0.60 (Rasch).

\paragraph{Simple baselines.} Rotation-0 letter shares, which need no model fit, flag three models whose most frequent letter
exceeds half of their replies (Table~\ref{tab:baselines}). These shares
count unparseable replies; $\pi_m(\mathrm{A})=0.90$ of Llama-3.2-1B
counts only guesses. Correlations between models' 2PL residuals
(observed correctness minus fitted probability, over the 391 split
all-4-correct questions) rank the Phi and Llama pairs first and second
of 276.

\paragraph{Other scores.} Counting a question correct in at least one,
two or three rotations, or by the mean, supports no conclusion. At ``at least two'' the 2PL null is rejected ($p_M=0.002$),
not the 3PL or SP null ($0.86$, $0.85$); at ``at least three'' Rasch,
KG-l and the 2PL are ($p_M$ 0.008--0.010), not KG-u ($0.28$). They mean little: the test rejects a
refitted 2PL null in 98.5\% and 58\% (at least two, three) of datasets
from the one-skill SP fit.
``At least one'' gives $p_M=0.996$ (SP null); the mean score has no
adequate null.

\section{Selective prediction (post hoc)}
\label{app:selective}

For Qwen2.5-7B, agreement with two large models from other families
predicts its MMLU errors better than 100 of its own samples, but a
typical pair of peers does not do so significantly; over five target models the gain varies,
and it is absent for Phi-3-mini. An earlier
version of this work reported this comparison with scores that used the
gold labels; we recompute it without labels, from a specification
written before the run.

\paragraph{Scores.} The target is whether a model's greedy rotation-0
answer is correct. Self-consistency with $K$ samples ($T=1$) scores a
question by the share of samples that give the most frequent answer. A
peer score is the share of $k$ other models whose greedy answer matches
the target's. The peer rule uses no labels: one model per other family,
largest first, which for Qwen2.5-7B gives Mistral-Small-22B and
Gemma-2-9B. AUROC is the probability that a correct answer
scores above a wrong one, ties counted as one half. Cost is in Qwen2.5-7B-equivalent forward passes ($B/7$ per
greedy pass of a $B$-parameter model).

\paragraph{Results.} For Qwen2.5-7B (greedy accuracy 0.71),
self-consistency with 100 samples reaches AUROC 0.653 (95\% bootstrap
interval 0.609--0.698). The two peers reach 0.774 (0.729--0.818), a
gain of 0.12 (0.07--0.17; DeLong test~\citep{delong1988} $p=10^{-5}$) at about 1/22 of the
cost; all 23 peers reach 0.758. For Qwen2.5-7B a typical pair does not: the
mean over all 253 pairs is 0.671, not significantly above
self-consistency, and 16\% of pairs are significantly above it. A shared letter preference
does not explain the gain: with agreement measured in rotations 1--3,
the mean AUROC is 0.788. Over the other targets the two-peer score beats self-consistency
for Mistral-7B v0.3 (0.820 against 0.699), SmolLM2-1.7B (0.869 against
0.677) and Llama-3-8B (0.813 against 0.759), but not for Phi-3-mini
(0.801 against 0.809); for Mistral-7B v0.3 and SmolLM2-1.7B even a
typical pair beats self-consistency. The evidence is one benchmark, 500 questions and
five targets.

\end{document}